\documentclass[letterpaper]{article}
\usepackage{amsmath,amssymb,graphicx}
\usepackage{aaai18}  
\usepackage{times}  
\usepackage{helvet}  
\usepackage{courier}  
\usepackage{url}  
\usepackage{graphicx}  
\usepackage{color}
\begin{document}
%

\newtheorem{theorem}{Theorem}
\newtheorem{lemma}{Lemma}
\newtheorem{claim}{Claim}
\newtheorem{proposition}{Proposition}
\newtheorem{definition}{Definition}
\newtheorem{corollary}{Corollary}
\renewcommand{\phi}{\varphi}
\renewcommand{\epsilon}{\varepsilon}
\newcommand{\<}{\langle}
\renewcommand{\>}{\rangle}
\newenvironment{proof}{\noindent{\sc Proof.}}{\hfill $\boxtimes\hspace{2mm}$\linebreak}
\newcommand{\qed}{\hfill $\boxtimes\hspace{1mm}$}

\newenvironment{proof-of-claim}{\noindent{\sc Proof of Claim.}}{\hfill $\boxtimes\hspace{2mm}$\linebreak}

\renewcommand{\H}{{\sf H}}
\newcommand{\K}{{\sf K}}
\newcommand{\N}{{\sf N}}
\newcommand{\B}{{\sf B}}
\newcommand{\cN}{{\sf \overline{N}}}
\newcommand{\cR}{{\sf \overline{R}}}
\newcommand{\KR}{\,\mbox{\scalebox{.75}{\framebox(10,10){$\sf R$}}}\,}
\newcommand{\KN}{\,\mbox{\scalebox{.75}{\framebox(10,10){$\sf N$}}}\,}
\renewcommand{\Box}{\,\mbox{\scalebox{.75}{\framebox(10,10){$ $}}}\,}
\newcommand{\cKN}{\overline{\mbox{\scalebox{.75}{\framebox(10,10){$\sf N$}}}}}
\newcommand{\R}{{\sf R}}
\newcommand{\SSS}{{\sf S}}
\newcommand{\M}{{\sf M}}
\newcommand{\A}{{\sf A}}

\newsavebox{\diamonddotsavebox}
\sbox{\diamonddotsavebox}{$\Diamond$\hspace{-1.8mm}\raisebox{0.3mm}{$\cdot$}\hspace{1mm}}
\newcommand{\diamonddot}{\usebox{\diamonddotsavebox}}

\title{Knowledge and Responsibility in Strategic Games}
\title{Knowledge and Responsibility of Strategic Coalitions}
\title{Blameworthiness in Strategic Games}

\author{Pavel Naumov \\Department of Mathematical Sciences\\  Claremont McKenna College\\Claremont, California 91711\\pnaumov@cmc.edu
\And  Jia Tao \\Department of Computer Science\\Lafayette College\\Easton, Pennsylvania 18042\\taoj@lafayette.edu}

\maketitle

\begin{abstract}
There are multiple notions of coalitional responsibility. The focus of this paper is on the blameworthiness defined through the principle of alternative possibilities: a coalition is blamable for a statement if the statement is true, but the coalition had a strategy to prevent it. The main technical result is a sound and complete bimodal logical system that describes properties of blameworthiness in one-shot games.
\end{abstract}


\section{Introduction}

It was a little after 9am on Friday, July 20th 2018, when a four-year-old boy accidentally shot his two-year old cousin in the town of Muscoy in Southern California. The victim was taken to a hospital, where she died an hour later~\cite{o18latimes}. The police arrested Cesar Lopez, victim's grandfather, as a felon in possession of a firearm and for child endangerment~\cite{jm18abc7}.

The first charge against Lopez, a previously convicted felon, is based on California Penal Code \S 29800 (a) (1) that prohibits firearm access to ``any person who has been convicted of, or has an outstanding warrant for, a felony under the laws of the United States, the State of California, or any other state, government, or country...". We assume that Lopez knew that California state law bans him from owning a gun, but his actions guaranteed that he broke the law.

The second charge is different because Lopez clearly never intended for his granddaughter to be killed. He never took any actions that would force her death. Nevertheless, he is {\em blamed} for not taking an action (locking the gun) to prevent the tragedy. Blameworthiness is tightly connected to the legal liability for negligence~\cite{g04mur}.

We are interested in logical systems for reasoning about different forms of responsibility.  
Xu~(\citeyear{x98jpl}) introduced a complete axiomatization of a modal logical system for reasoning about responsibility defined as taking actions that guarantee a certain outcome. In our example,  by possessing a gun Lopez guaranteed that he was responsible for braking California law. Broersen,  Herzig, and Troquard~(\citeyear{bht09jancl}) extended Xu's work from individual responsibility to group responsibility. 
In this paper we propose a complete logical system for reasoning about another form of responsibility that we call blameworthiness: a coalition is blamable for an outcome $\phi$ if $\phi$ is true, but the coalition had a strategy to prevent $\phi$. In our example, Lopez had a strategy to prevent the death by keeping the gun in a safe place.




\subsubsection{Principle of Alternative Possibilities}

Throughout centuries, blameworthiness, especially in the context of free will and moral responsibility, has been at the focus of philosophical discussions~\cite{se13eb}. Modern works on this topic include~\cite{f94p,fr00,nk07nous,m15ps,w17}. Frankfurt~(\citeyear{f69tjop}) acknowledges that a dominant role in these discussions has been played by what he calls a {\em principle of alternate possibilities}: ``a person is morally responsible for what he has done only if he could have done otherwise". As with many general principles, this one has many limitations  that Frankfurt discusses; for example, when a person is coerced into doing something. Following the established tradition~\cite{w17}, we refer to this principle as the principle of {\em alternative} possibilities. Cushman~(\citeyear{c15cop}) talks about {\em counterfactual possibility}: ``a person could have prevented their harmful conduct, even though they did not."  

Halpern and Pearl proposed several versions of a formal definition of causality as a relation between sets of variables~\cite{h16}. This definition uses the counterfactual requirement which formalizes the principle of alternative possibilities. Halpern and Kleiman-Weiner~(\citeyear{hk18aaai}) used a similar setting to define {\em degrees} of blameworthiness. Batusov and Soutchanski~(\citeyear{bs18aaai}) gave a counterfactual-based definition of causality in situation calculus. 

\subsubsection{Coalitional Power in Strategic Games}

Pauly~(\citeyear{p01illc,p02}) introduced logics of coalitional power that can be used to describe group abilities to achieve a certain result. His approach has been widely studied in the literature~\cite{g01tark,vw05ai,b07ijcai,sgvw06aamas,abvs10jal,avw09ai,b14sr,gjt13jaamas,alnr11jlc,ga17tark,ge18aamas}.

In this paper we use Marc Pauly's framework to define blameworthiness of coalitions of players in strategic (one-shot) games. We say that a coalition $C$ could be blamed for an outcome $\phi$ if $\phi$ is true, but the coalition $C$ had a strategy to prevent $\phi$. Thus, just like Halpern and Pearl's  formal definition of causality, our definition of blameworthiness is based on the principle of alternative possibilities. However, because Marc Pauly's framework separates agents and outcomes, the proposed definition of blameworthiness is different and, arguably, more succinct. 

The main technical result of this paper is a sound and complete bimodal logical system describing the interplay between group blameworthiness modality and necessity (or universal truth) modality. Our system is significantly different from earlier mentioned axiomatizations \cite{x98jpl} and \cite{bht09jancl} because our semantics incorporates the principle of alternative possibilities.

\subsubsection{Paper Outline}

This paper is organized as follows. First, we introduce the formal syntax and semantics of our logical system. Next, we state and discuss its axioms. In the section that follows, we give examples of formal derivations in our system. In the next two sections we prove the soundness and the completeness. The last section concludes with a discussion of possible future work.

\section{Syntax and Semantics}\label{syntax and semantics section}

In this paper we assume a fixed set $\mathcal{A}$ of agents and a fixed set of propositional variables $\sf Prop$. By a coalition we mean an arbitrary subset of set $\mathcal{A}$.
\begin{definition}\label{Phi}
$\Phi$ is the minimal set of formulae such that
\begin{enumerate}
    \item $p\in\Phi$ for each variable $p\in {\sf Prop}$,
    \item $\phi\to\psi,\neg\phi\in\Phi$ for all formulae $\phi,\psi\in\Phi$,
    \item $\N\phi$, $\B_C\phi\in\Phi$ for each coalition $C\subseteq\mathcal{A}$ and each formula $\phi\in\Phi$. 
\end{enumerate}
\end{definition}
In other words, language $\Phi$ is defined by  grammar:
$$
\phi := p\;|\;\neg\phi\;|\;\phi\to\phi\;|\;\N\phi\;|\;\B_C\phi.
$$
Formula $\N\phi$ is read as ``statement $\phi$ is true under each play" and formula $\B_C\phi$ as ``coalition $C$ is blamable for $\phi$".

Boolean connectives $\vee$, $\wedge$, and $\leftrightarrow$ as well as constants $\bot$ and $\top$ are defined in the standard way. By formula $\cN\phi$ we mean $\neg\N\neg\phi$. For the disjunction of multiple formulae, we assume that parentheses are nested to the left. That is, formula $\chi_1\vee\chi_2\vee\chi_3$ is a shorthand for $(\chi_1\vee\chi_2)\vee\chi_3$. As usual, the empty disjunction is defined to be $\bot$. For any two sets $X$ and $Y$, by $X^Y$ we denote the set of all functions from $Y$ to $X$.

The formal semantics of modalities $\N$ and  $\B$ is defined in terms of models, which we call \emph{games}. 

\begin{definition}\label{game definition}
A game is a tuple $\left(\Delta,\Omega,P,\pi\right)$, where 
\begin{enumerate}
    \item $\Delta$ is a nonempty set of ``actions",
    \item $\Omega$ is a set of ``outcomes",
    \item the set of ``plays" $P$ is an arbitrary set of pairs $(\delta,\omega)$ such that $\delta\in\Delta^\mathcal{A}$ and $\omega\in\Omega$, 
    \item $\pi$ is a function that maps $\sf Prop$ into subsets of $P$.
\end{enumerate}
\end{definition} 
The example from the introduction can be captured in our setting by assuming that Lopez is the only actor who has two possible actions: $hide$ and $expose$ the gun in the game with two outcomes $alive$ and $dead$. Although a complete action profile is a function from the set of all agents to the domain of actions, in a single agent case any such profile can be described by specifying just the action of the single player. Thus, by complete action profile $hide$ we mean action profile that maps agent Lopez into action $hide$. The set of possible plays of this game consists of pairs $\{(hide,alive),(expose,alive),(expose,dead)\}$.

The above definition of a game is very close but not identical to the definition of a game frame in Pauly~(\citeyear{p01illc,p02}) and the definition of a concurrent game structure, the semantics of ATL~\cite{ahk02}. Unlike these works, here we assume that the domain of choices is the same for all states and all agents. This difference is insignificant because all domains of choices in a game frame/concurrent game structure could be replaced with their union. More importantly, we assume that the mechanism is a relation, not a function. Our approach is more general, as it allows us to talk about blameworthiness in nondeterministic games, it also results in fewer axioms. Also, we do {\em not} assume that for any complete action profile $\delta$ there is at least one outcome $\omega$ such that $(\delta,\omega)\in P$. Thus, we allow the system to terminate under some action profiles without reaching an outcome. Without this assumption, we would need to add one extra axiom: $\neg\B_C\bot$ and to make minor changes in the  proof of the completeness.

Finally, in this paper we assume that atomic propositions are interpreted as statements about plays, not just outcomes. For example, the meaning of an atomic proposition $p$ could be statement ``either Lopez locked his gun or his granddaughter is dead". This is a more general approach than the one used in the existing literature, where atomic propositions are usually interpreted as statements about just outcomes. This difference is formally captured in the above definition through the assumption that value of $\pi$ is a set of plays, not just a set of outcomes. As a result of this more general approach, all other statements in our logical system are also statements about plays, not outcomes. This is why relation $\Vdash$ in Definition~\ref{sat} has a play (not an outcome) on the left. 

If $s_1$ and $s_2$ are action profiles of coalitions $C_1$ and $C_2$, respectively, and $C$ is any coalition such that $C\subseteq C_1\cap C_2$, then we write $s_1=_C s_2$ to denote that $s_1(a)=s_2(a)$ for each agent $a\in C$.

Next is the key definition of this paper. Its item 5 formally specifies blameworthiness using the principle of alternative possibilities.

\begin{definition}\label{sat} 
For any play $(\delta,\omega)\in P$ of a game
$\left(\Delta,\Omega,P,\pi\right)$ and any formula $\phi\in\Phi$, the satisfiability relation $(\delta,\omega)\Vdash\phi$ is defined recursively as follows:
\begin{enumerate}
    \item $(\delta,\omega)\Vdash p$ if $(\delta,\omega)\in \pi(p)$, where $p\in {\sf Prop}$,
    \item $(\delta,\omega)\Vdash \neg\phi$ if $(\delta,\omega)\nVdash \phi$,
    \item $(\delta,\omega)\Vdash\phi\to\psi$ if $(\delta,\omega)\nVdash\phi$ or $(\delta,\omega)\Vdash\psi$,
    \item $(\delta,\omega)\Vdash\N\phi$ if $(\delta',\omega')\Vdash\phi$ for each play $(\delta',\omega')\in P$,
    \item $(\delta,\omega)\Vdash\B_C\phi$ if $(\delta,\omega)\Vdash\phi$ and there is $s\in \Delta^C$ such that for each play $(\delta',\omega')\in P$, if $s=_C\delta'$, then $(\delta',\omega')\nVdash\phi$.
\end{enumerate}
\end{definition}

\section{Axioms}\label{axioms section}

In addition to the propositional tautologies in  language $\Phi$, our logical system contains the following axioms.

\begin{enumerate}
    \item Truth: $\N\phi\to\phi$ and $\B_C\phi\to\phi$,
    \item Distributivity: $\N(\phi\to\psi)\to(\N\phi\to \N\psi)$,
    \item Negative Introspection: $\neg\N\phi\to\N\neg\N\phi$,
    \item None to Blame: $\neg\B_\varnothing\phi$,
    \item Joint Responsibility: if $C\cap D=\varnothing$, then\\ $\cN\B_C\phi\wedge\cN\B_D\psi\to (\phi\vee\psi\to\B_{C\cup D}(\phi\vee\psi))$,
    \item Blame for Cause: $\N(\phi\to\psi)\to(\B_C\psi\to(\phi\to \B_C\phi))$,
    \item Monotonicity: $\B_C\phi\to\B_D\phi$, where $C\subseteq D$,
    \item Fairness: $\B_C\phi\to\N(\phi\to\B_C\phi)$.
\end{enumerate}

We write $\vdash\phi$ if formula $\phi$ is provable from the axioms of our system using the Modus Ponens and
the Necessitation inference rules:
$$
\dfrac{\phi,\phi\to\psi}{\psi},
\hspace{20mm}
\dfrac{\phi}{\N\phi}.
$$
We write $X\vdash\phi$ if formula $\phi$ is provable from the theorems of our logical system and an additional set of axioms $X$ using only the Modus Ponens inference rule.

The Truth axiom for modality $\N$, the Distributivity axiom, and the Negative Introspection axiom together with the Necessitation inference rule capture the fact that modality $\N$, per Definition~\ref{sat}, is an S5 modality and thus satisfies all standard S5 properties. 

The Truth axiom for modality $\B$ states that any coalition can be blamed only for a statement which is true. The None to Blame axiom states that the empty coalition cannot be blamed for anything. Intuitively, this axiom is true because the empty coalition has no power to prevent anything.

The Joint Responsibility axiom states that if disjoint coalitions $C$ and $D$ can be blamed for statements $\phi$ and $\psi$, respectively, on {\em some other (possibly two different) plays of the game} and the disjunction $\phi\vee\psi$ is true on the current play, then the union of the two coalitions can be blamed for this disjunction on the current play. This axiom remotely resembles Xu~(\citeyear{x98jpl}) axiom for independence of individual agents, which in our notations can be stated as
$$
\cN\B_{a_1}\phi_1\wedge\dots\wedge\cN\B_{a_n}\phi_n\to \cN(\B_{a_1}\phi_1\wedge\dots\wedge\B_{a_n}\phi_n).
$$
Broersen, Herzig, and Troquard~(\citeyear{bht09jancl}) captured the independence of disjoint coalitions $C$ and $D$ in their Lemma 17:
$$
\cN\B_C\phi\wedge\cN\B_D\psi\to\cN(\B_C\phi\wedge\B_D\psi).
$$
In spite of these similarities, the definition of responsibility used in \cite{x98jpl} and \cite{bht09jancl} does not assume the principle of alternative possibilities.
The Joint Responsibility axiom is also similar to Marc Pauly~(\citeyear{p01illc,p02}) Cooperation axiom for logic of coalitional power:
$$
\SSS_C\phi\wedge\SSS_D\psi\to\SSS_{C\cup D}(\phi\wedge\psi),
$$
where coalitions $C$ and $D$ are disjoint and $\SSS_C\phi$ stands for ``coalition $C$ has a strategy to achieve $\phi$".

The Blame for Cause axiom states that if formula $\phi$ universally implies $\psi$ (informally, $\phi$ is a ``cause" of $\psi$), then any coalition blamable for $\psi$ should also be blamable for the ``cause" $\phi$ as long as $\phi$ is actually true.
The Monotonicity axiom states that any coalition is blamed for anything that a subcoalition is blamed for.
Finally, the Fairness axiom states that if a coalition $C$ is blamed for $\phi$, then it should be blamed for $\phi$ whenever $\phi$ is true.

\section{Examples of Derivations}\label{examples section}

The soundness of the axioms of our logical system is established in the next section.
In this section we give several examples of formal proofs in our system. Together with the Truth axiom, the first example shows that statements $\B_C\B_C\phi$ and $\B_C\phi$  are equivalent in our system. That is, coalition $C$ can be blamed for being blamed for $\phi$ if and only if it can be blamed for $\phi$.

\begin{lemma}\label{nested blame lemma}
$\vdash \B_C\phi\to\B_C\B_C\phi$.
\end{lemma}
\begin{proof}
Note that $\vdash\B_C\phi\to\phi$ by the Truth axiom. Thus, $\vdash\N(\B_C\phi\to\phi)$ by the Necessitation rule. At the same time, $$\vdash\N(\B_C\phi\to\phi)\to(\B_C\phi\to(\B_C\phi\to\B_C\B_C\phi))$$ is an instance of the Blame for Cause axiom. Then, $\vdash\B_C\phi\to(\B_C\phi\to\B_C\B_C\phi)$ by the Modus Ponens inference rule. Therefore, $\vdash \B_C\phi\to\B_C\B_C\phi$ by the propositional reasoning. 
\end{proof}

The rest of the examples in this section are used later in the proof of the completeness.

\begin{lemma}\label{alt fairness lemma}
$\vdash\cN\B_C\phi\to(\phi\to\B_C\phi)$.
\end{lemma}
\begin{proof}
Note that $\vdash \B_C\phi\to\N(\phi\to\B_C\phi)$ by the Fairness axiom. Hence, $\vdash \neg\N(\phi\to\B_C\phi)\to \neg\B_C\phi$, by the law of contrapositive. Thus, $\vdash \N(\neg\N(\phi\to\B_C\phi)\to \neg\B_C\phi)$ by the Necessitation inference rule. Hence, by the Distributivity axiom and the Modus Ponens inference rule,
$$\vdash \N\neg\N(\phi\to\B_C\phi)\to \N\neg\B_C\phi.$$ At the same time, by the Negative Introspection axiom:
$$
\vdash \neg\N(\phi\to\B_C\phi)\to\N\neg\N(\phi\to\B_C\phi).
$$
Thus, by the laws of propositional reasoning,
$$\vdash \neg\N(\phi\to\B_C\phi)\to \N\neg\B_C\phi.$$
Hence, by the law of contrapositive,
$$\vdash \neg\N\neg\B_C\phi\to \N(\phi\to\B_C\phi).$$
Note that $\N(\phi\to\B_C\phi)\to(\phi\to\B_C\phi)$ is an instance of the Truth axiom. Thus, by propositional reasoning,
$$\vdash \neg\N\neg\B_C\phi\to (\phi\to\B_C\phi).$$
Hence, $\vdash \cN\B_C\phi\to (\phi\to\B_C\phi)$ by the definition of $\cN$.
\end{proof}

\begin{lemma}\label{alt cause lemma}
If $\vdash \phi\leftrightarrow \psi$, then $\vdash \B_C\phi\to\B_C\psi$.
\end{lemma}
\begin{proof}
By the Blame for Cause axiom,
$$
\vdash \N(\psi\to\phi)\to(\B_C\phi\to(\psi\to \B_C\psi)).
$$
Assumption  $\vdash \phi\leftrightarrow \psi$ implies $\vdash \psi\to \phi$ by the laws of propositional reasoning. Thus, $\vdash \N(\psi\to \phi)$ by the Necessitation inference rule. Hence, by the Modus Ponens rule,
$$
\vdash \B_C\phi\to(\psi\to \B_C\psi).
$$
Thus, by the laws of propositional reasoning,
\begin{equation}\label{sofia}
\vdash (\B_C\phi\to\psi)\to (\B_C\phi\to \B_C\psi).
\end{equation}
Note that $\vdash \B_C\phi\to\phi$ by the Truth axiom. At the same time, $\vdash \phi\leftrightarrow \psi$ by the assumption of the lemma. Thus, by the laws of propositional reasoning, $\vdash \B_C\phi\to\psi$. Therefore,
$
\vdash \B_C\phi\to \B_C\psi
$
by the Modus Ponens inference rule from statement~(\ref{sofia}).
\end{proof}

\begin{lemma}\label{add cN lemma}
$\phi\vdash \cN\phi$.
\end{lemma}
\begin{proof}
By the Truth axioms, $\vdash\N\neg\phi\to\neg\phi$. Thus, by the law of contrapositive, $\vdash\phi\to \neg\N\neg\phi$. Hence, $\vdash\phi\to \cN\phi$ by the definition of the modality $\cN$. Therefore, $\phi\vdash \cN\phi$ by the Modus Ponens inference rule.
\end{proof}

The next lemma generalizes the Joint Responsibility axiom from two coalitions to multiple coalitions.
\begin{lemma}\label{super joint responsibility lemma}
For any integer $n\ge 0$ and any pairwise disjoint sets $D_1,\dots,D_n$,
$$
\{\cN\B_{D_i}\chi_i\}_{i=1}^n,\chi_1\vee\dots\vee\chi_n
\vdash \B_{D_1\cup\dots\cup D_n}(\chi_1\vee \dots\vee\chi_n).
$$
\end{lemma}
\begin{proof}
We prove the lemma by induction on $n$. If $n=0$, then disjunction $\chi_1\vee\dots\vee \chi_n$ is Boolean constant false $\bot$ by definition. Thus, the statement of the lemma is $\bot\vdash\B_\varnothing\bot$, which is provable in the propositional logic due to the assumption $\bot$ on the left-hand side of $\vdash$.

Next, suppose that $n=1$. Then, from Lemma~\ref{alt fairness lemma} it follows that
$\cN\B_{D_1}\chi_1,\chi_1\vdash\B_{D_1}\chi_1$.

Suppose that $n\ge 2$. By the Joint Responsibility axiom and the Modus Ponens inference rule,
\begin{eqnarray*}
&&\hspace{-8mm}\cN\B_{D_1\cup \dots \cup D_{n-1}}(\chi_1\vee\dots\vee\chi_{n-1}),
\cN\B_{D_n}\chi_n,\\
&&\hspace{-8mm}\chi_1\vee\dots\vee\chi_{n-1}\vee\chi_n\\
&&\vdash \B_{D_1\cup \dots \cup D_{n-1}\cup D_n}(\chi_1\vee\dots\vee\chi_{n-1}\vee \chi_n).
\end{eqnarray*}
Thus, by Lemma~\ref{add cN lemma},
\begin{eqnarray*}
&&\hspace{-8mm}\B_{D_1\cup \dots \cup D_{n-1}}(\chi_1\vee\dots\vee\chi_{n-1}),
\cN\B_{D_n}\chi_n,\\
&&\hspace{-8mm}\chi_1\vee\dots\vee\chi_{n-1}\vee\chi_n\\
&&\vdash \B_{D_1\cup \dots \cup D_{n-1}\cup D_n}(\chi_1\vee\dots\vee\chi_{n-1}\vee \chi_n).
\end{eqnarray*}
At the same time, by the induction hypothesis,
\begin{eqnarray*}
&&\hspace{-8mm}\{\cN\B_{D_i}\chi_i\}_{i=1}^{n-1},\chi_1\vee\dots\vee\chi_{n-1}\\
&&\vdash \B_{D_1\cup\dots\cup D_{n-1}}(\chi_1\vee \dots\vee\chi_{n-1}).
\end{eqnarray*}
Hence,
\begin{eqnarray*}
&&\hspace{-5mm}\{\cN\B_{D_i}\chi_i\}_{i=1}^n,\chi_1\vee\dots\vee\chi_{n-1},\chi_1\vee\dots\vee\chi_{n-1}\vee\chi_n\\
&&\vdash \B_{D_1\cup\dots\cup D_{n-1}\cup D_n}(\chi_1\vee \dots\vee\chi_{n-1}\vee\chi_n).
\end{eqnarray*}
Since $\chi_1\vee\dots\vee\chi_{n-1}\vdash\chi_1\vee\dots\vee\chi_{n-1}\vee\chi_n$ is provable in propositional logic,
\begin{eqnarray}
&&\hspace{-10mm}\{\cN\B_{D_i}\chi_i\}_{i=1}^n,\chi_1\vee\dots\vee\chi_{n-1}\nonumber\\
&&\hspace{0mm} \vdash \B_{D_1\cup\dots\cup D_{n-1}\cup D_n}(\chi_1\vee \dots\vee\chi_{n-1}\vee\chi_n).\label{part 1}
\end{eqnarray}
Similarly, by the Joint Responsibility axiom and the Modus Ponens inference rule,
\begin{eqnarray*}
&&\hspace{-8mm}\cN\B_{D_1}\chi_1,\cN\B_{D_2\cup \dots \cup D_n}(\chi_2\vee\dots\vee\chi_n),\\
&&\hspace{-8mm}\chi_1\vee(\chi_2\vee\dots\vee\chi_n)\\
&&\vdash \B_{D_1\cup \dots \cup D_{n-1}\cup D_n}(\chi_1\vee(\chi_2\vee\dots\vee \chi_n)).
\end{eqnarray*}
Since formula 
$\chi_1\vee(\chi_2\vee\dots\vee \chi_n)\leftrightarrow \chi_1\vee\chi_2\vee\dots\vee \chi_n$ is provable in the propositional logic, by Lemma~\ref{alt cause lemma},
\begin{eqnarray*}
&&\hspace{-7mm}\cN\B_{D_1}\chi_1,\cN\B_{D_2\cup \dots \cup D_n}(\chi_2\vee\dots\vee\chi_n),\chi_1\vee\chi_2\vee\dots\vee\chi_n\\
&&\vdash \B_{D_1\cup \dots \cup D_{n-1}\cup D_n}(\chi_1\vee\chi_2\vee\dots\vee \chi_n).
\end{eqnarray*}
Thus, by Lemma~\ref{add cN lemma},
\begin{eqnarray*}
&&\hspace{-7mm}\cN\B_{D_1}\chi_1,\B_{D_2\cup \dots \cup D_n}(\chi_2\vee\dots\vee\chi_n),\chi_1\vee\chi_2\vee\dots\vee\chi_n\\
&&\vdash \B_{D_1\cup \dots \cup D_{n-1}\cup D_n}(\chi_1\vee\chi_2\vee\dots\vee \chi_n).
\end{eqnarray*}
At the same time, by the induction hypothesis,
$$
\{\cN\B_{D_i}\chi_i\}_{i=2}^n,\chi_2\vee\dots\vee\chi_n
\vdash \B_{D_2\cup\dots\cup D_n}(\chi_2\vee \dots\vee\chi_n).
$$
Hence,
\begin{eqnarray*}
&&\hspace{-8mm}\{\cN\B_{D_i}\chi_i\}_{i=1}^n,\chi_2\vee\dots\vee\chi_n,\chi_1\vee\chi_2\vee\dots\vee\chi_n\\
&&\vdash \B_{D_1\cup D_2\cup\dots\cup D_n}(\chi_1\vee\chi_2\vee\dots\vee\chi_n).
\end{eqnarray*}
Since $\chi_2\vee\dots\vee\chi_{n}\vdash\chi_1\vee\dots\vee\chi_{n-1}\vee\chi_n$ is provable in propositional logic,
\begin{eqnarray}
&&\hspace{-10mm}\{\cN\B_{D_i}\chi_i\}_{i=1}^n,\chi_2\vee\dots\vee\chi_n\nonumber\\
&&\hspace{0mm} \vdash \B_{D_1\cup\dots\cup D_{n-1}\cup D_n}(\chi_1\vee\chi_2\vee\dots\vee\chi_n).\label{part 2}
\end{eqnarray}
Finally, note that the following statement is provable in the propositional logic for $n\ge 2$,
$$
\vdash\chi_1\vee\dots\vee\chi_n\to(\chi_1\vee\dots\vee\chi_{n-1})\vee 
(\chi_2\vee\dots\vee\chi_n).
$$
Therefore, from statement~(\ref{part 1}) and statement~(\ref{part 2}), 
$$
\{\cN\B_{D_i}\chi_i\}_{i=1}^n,\chi_1\vee\dots\vee\chi_n
\vdash \B_{D_1\cup\dots\cup D_n}(\chi_1\vee \dots\vee\chi_n)
$$
by the laws of propositional reasoning.
\end{proof}

\begin{lemma}\label{super distributivity}
If $\phi_1,\dots,\phi_n\vdash\psi$, then $\N\phi_1,\dots,\N\phi_n\vdash\N\psi$.
\end{lemma}
\begin{proof}
By the deduction lemma applied $n$ times, assumption $\phi_1,\dots,\phi_n\vdash\psi$ implies that
$$
\vdash\phi_1\to(\phi_2\to\dots(\phi_n\to\psi)\dots).
$$
Hence, by the Necessitation inference rule,
$$
\vdash\N(\phi_1\to(\phi_2\to\dots(\phi_n\to\psi)\dots)).
$$
Thus, by the Distributivity axiom and the Modus Ponens,
$$
\vdash\N\phi_1\to\N(\phi_2\to\dots(\phi_n\to\psi)\dots).
$$
Hence, by the Modus Ponens inference rule,
$$
\N\phi_1\vdash\N(\phi_2\to\dots(\phi_n\to\psi)\dots).
$$
Therefore, by applying the previous steps $(n-1)$ more times,
$\N\phi_1,\dots,\N\phi_n\vdash\N\psi$.
\end{proof}

\begin{lemma}\label{positive introspection lemma}
$\vdash \N\phi\to\N\N\phi$.
\end{lemma}
\begin{proof}
Formula $\N\neg\N\phi\to\neg\N\phi$ is an instance of the Truth axiom. Thus, $\vdash \N\phi\to\neg\N\neg\N\phi$ by contraposition. Hence, taking into account the following instance of  the Negative Introspection axiom: $\neg\N\neg\N\phi\to\N\neg\N\neg\N\phi$,
we have 
\begin{equation}\label{pos intro eq 2}
\vdash \N\phi\to\N\neg\N\neg\N\phi.
\end{equation}

At the same time, $\neg\N\phi\to\N\neg\N\phi$ is an instance of the Negative Introspection axiom. Thus, $\vdash \neg\N\neg\N\phi\to \N\phi$ by the law of contrapositive in the propositional logic. Hence, by the Necessitation inference rule, 
$\vdash \N(\neg\N\neg\N\phi\to \N\phi)$. Thus, by  the Distributivity axiom and the Modus Ponens inference rule, 
$
   \vdash \N\neg\N\neg\N\phi\to \N\N\phi.
$
 The latter, together with statement~(\ref{pos intro eq 2}), implies the statement of the lemma by propositional reasoning.
\end{proof}

\begin{lemma}\label{five plus plus}
For any integer $n\ge 0$ and any disjoint sets $D_1,\dots,D_n\subseteq C$,
$$
\{\cN\B_{D_i}\chi_i\}_{i=1}^n,\N(\phi\to\chi_1\vee\dots\vee\chi_n)\vdash\N(\phi\to\B_C\phi).
$$
\end{lemma}
\begin{proof}
By Lemma~\ref{super joint responsibility lemma},
$$
\{\cN\B_{D_i}\chi_i\}_{i=1}^n,\chi_1\vee\dots\vee\chi_n\vdash \B_{D_1\cup\dots\cup D_n}(\chi_1\vee\dots\vee\chi_n).
$$
Thus, by the Monotonicity axiom, 
$$
\{\cN\B_{D_i}\chi_i\}_{i=1}^n,\chi_1\vee\dots\vee\chi_n\vdash \B_{C}(\chi_1\vee\dots\vee\chi_n).
$$
Hence, by the Modus Ponens inference rule
$$
\{\cN\B_{D_i}\chi_i\}_{i=1}^n,\phi,\phi\to\chi_1\vee\dots\vee\chi_n\vdash \B_C(\chi_1\vee\dots\vee\chi_n).
$$
By the Truth axiom and the Modus Ponens inference rule,
$$
\{\cN\B_{D_i}\chi_i\}_{i=1}^n,\phi,\N(\phi\to\chi_1\vee\dots\vee\chi_n)\vdash \B_C(\chi_1\vee\dots\vee\chi_n).
$$
Note that $\N(\phi\to\chi_1\vee\dots\vee\chi_n)\to(\B_C(\chi_1\vee\dots\vee\chi_n)\to(\phi\to\B_C\phi))$ is an instance of the Blame for Cause axiom. Thus, by the Modus Ponens inference rule applied twice,
$$
\{\cN\B_{D_i}\chi_i\}_{i=1}^n,\phi,\N(\phi\to\chi_1\vee\dots\vee\chi_n)\vdash\phi\to\B_C\phi.
$$
By the Modus Ponens inference rule,
$$
\{\cN\B_{D_i}\chi_i\}_{i=1}^n,\phi, \N(\phi\to\chi_1\vee\dots\vee\chi_n)\vdash\B_C\phi.
$$
By the deduction lemma,
$$
\{\cN\B_{D_i}\chi_i\}_{i=1}^n,\N(\phi\to\chi_1\vee\dots\vee\chi_n)\vdash\phi\to\B_C\phi.
$$
By Lemma~\ref{super distributivity},
$$
\{\N\cN\B_{D_i}\chi_i\}_{i=1}^n,\N\N(\phi\to\chi_1\vee\dots\vee\chi_n)\vdash\N(\phi\to\B_C\phi).
$$
By the definition of modality $\cN$, the Negative Introspection axiom, and the Modus Ponens inference rule,
$$
\{\cN\B_{D_i}\chi_i\}_{i=1}^n,\N\N(\phi\to\chi_1\vee\dots\vee\chi_n)\vdash\N(\phi\to\B_C\phi)
$$
Therefore, by Lemma~\ref{positive introspection lemma} and the Modus Ponens inference rule, the statement of the lemma follows. 
\end{proof}

\section{Soundness}\label{soundness section}
In the following lemmas, $(\delta,\omega)\in P$ is a play of an arbitrary game $(\Delta,\Omega,P,\pi)$ and $\phi,\psi\in \Phi$ are arbitrary formulae. 

\begin{lemma}
$(\delta,\omega)\nVdash \B_\varnothing\phi$. 
\end{lemma}
\begin{proof}
Suppose that $(\delta,\omega)\Vdash \B_\varnothing\phi$. Thus, by Definition~\ref{sat}, we have $(\delta,\omega)\Vdash \phi$ and there is an action profile $s\in\Delta^\varnothing$ such that for each play $(\delta',\omega')\in P$, if $s=_\varnothing\delta'$, then $(\delta',\omega')\nVdash\phi$.

Consider $\delta'=\delta$ and $\omega'=\omega$. Note that $s=_\varnothing\delta'$ is vacuously true. Hence, $(\delta',\omega')\nVdash\phi$. In other words, $(\delta,\omega)\nVdash\phi$, which leads to a contradiction.
\end{proof}

\begin{lemma}
For all sets $C,D\subseteq\mathcal{A}$ 
such that $C\cap D=\varnothing$, if $(\delta,\omega)\Vdash \cN\B_C\phi$, $(\delta,\omega)\Vdash \cN\B_D\psi$, and $(\delta,\omega)\Vdash \phi\vee\psi$, then $(\delta,\omega)\Vdash \B_{C\cup D}(\phi\vee\psi)$.
\end{lemma}
\begin{proof}
Let $(\delta,\omega)\Vdash \cN\B_C\phi$ and $(\delta,\omega)\Vdash \cN\B_D\psi$. Thus, by Definition~\ref{sat} and the definition of modality $\cN$, there are plays $(\delta_1,\omega_1)\in P$ and $(\delta_2,\omega_2)\in P$ such that $(\delta_1,\omega_1)\Vdash \B_C\phi$ and $(\delta_2,\omega_2)\Vdash \B_D\psi$.

By Definition~\ref{sat}, statement $(\delta_1,\omega_1)\Vdash \B_C\phi$ implies that there is $s_1\in \Delta^C$ such that for each play $(\delta',\omega')\in P$, if $s_1=_C\delta'$, then $(\delta',\omega')\nVdash\phi$.

Similarly, by Definition~\ref{sat}, statement $(\delta_2,\omega_2)\Vdash \B_D\psi$ implies that there is $s_2\in \Delta^D$ such that for each play $(\delta',\omega')\in P$, if $s_2=_D\delta'$, then $(\delta',\omega')\nVdash\psi$.

Consider an action profile $s$ of coalition $C\cup D$ such that
$$
s(a)=
\begin{cases}
s_1(a), & \mbox{ if } a\in C,\\
s_2(a), & \mbox{ if } a\in D.
\end{cases}
$$
Note that the action profile $s$ is well-defined because sets $C$ and $D$ are disjoint by the assumption of the lemma. 

The choice of action profiles $s_1$, $s_2$, and $s$ implies that  for each play $(\delta',\omega')\in P$, if $s=_{C\cup D}\delta'$, then $(\delta',\omega')\nVdash\phi$ and $(\delta',\omega')\nVdash\psi$. 
Thus, for each play $(\delta',\omega')\in P$, if $s=_{C\cup D}\delta'$, then $(\delta',\omega')\nVdash\phi\vee\psi$. 
Therefore, $(\delta,\omega)\Vdash \B_{C\cup D}(\phi\vee\psi)$  by Definition~\ref{sat} and due to the assumption $(\delta,\omega)\Vdash \phi\vee\psi$ of the lemma.
\end{proof}

\begin{lemma}
If $(\delta,\omega)\Vdash \N(\phi\to\psi)$, $(\delta,\omega)\Vdash \B_C\psi$, and $(\delta,\omega)\Vdash \phi$, then $(\delta,\omega)\Vdash \B_C\phi$.
\end{lemma}
\begin{proof}
By Definition~\ref{sat}, assumption $(\delta,\omega)\Vdash \B_C\psi$ implies that there is $s\in \Delta^C$ such that for each play $(\delta',\omega')\in P$, if $s=_C\delta'$, then $(\delta',\omega')\nVdash\psi$. 

At the same time, $(\delta',\omega')\Vdash\phi\to\psi$ for each play $(\delta',\omega')\in P$ by  the assumption $(\delta,\omega)\Vdash \N(\phi\to\psi)$ of the lemma and Definition~\ref{sat}. 

Thus,  $(\delta',\omega')\nVdash\phi$ for each play $(\delta',\omega')\in P$ such that $s=_C\delta'$ by Definition~\ref{sat}. Hence, $(\delta,\omega)\Vdash \B_C\phi$ by Definition~\ref{sat} and the assumption $(\delta,\omega)\Vdash \phi$ of the lemma.
\end{proof}

\begin{lemma}
For all sets $C,D\in\mathcal{A}$ such that $C\subseteq D$, if $(\delta,\omega)\Vdash \B_C\phi$, then $(\delta,\omega)\Vdash \B_D\phi$.
\end{lemma}
\begin{proof}
By Definition~\ref{sat}, assumption $(\delta,\omega)\Vdash \B_C\phi$ implies that $(\delta,\omega)\Vdash\phi$ and there is $s\in \Delta^C$ such that for each play $(\delta',\omega')\in P$, if $s=_C\delta'$, then $(\delta',\omega')\nVdash\phi$. 

By Definition~\ref{game definition}, set $\Delta$ is not empty. Let $d_0\in\Delta$. Consider an action profile $s'$ of coalition $D$ such that
$$
s'(a)=
\begin{cases}
s(a), & \mbox{ if } a\in C,\\
d_0, & \mbox{ if } a\in D\setminus C.
\end{cases}
$$
Then, by the choice of action profile $s$ and because $C\subseteq D$, for each play $(\delta',\omega')\in P$, if $s'=_D\delta'$, then $(\delta',\omega')\nVdash\phi$. Therefore, $(\delta,\omega)\Vdash \B_D\phi$ by Definition~\ref{sat} and  because  $(\delta,\omega)\Vdash\phi$, as we have shown earlier.
\end{proof}

\begin{lemma}
If $(\delta,\omega)\Vdash \B_C\phi$, then $(\delta,\omega)\Vdash \N(\phi\to\B_C\phi)$.
\end{lemma}
\begin{proof} 
Consider any play $(\delta',\omega')\in P$. By Definition~\ref{sat}, it suffices to show that if $(\delta',\omega')\Vdash \phi$, then $(\delta',\omega')\Vdash \B_C\phi$. Thus, again by Definition~\ref{sat}, it suffices to prove there is $s\in \Delta^C$ such that for each play $(\delta'',\omega'')\in P$, if $s=_C\delta''$, then $(\delta'',\omega'')\nVdash\phi$. The last statement follows from the assumption $(\delta,\omega)\Vdash \B_C\phi$ and Definition~\ref{sat}.
\end{proof}

\section{Completeness}\label{completeness section}

We start the proof of
the completeness by defining the canonical game $G(\omega_0)=\left(\Delta,\Omega,P,\pi\right)$ for each maximal consistent set of formulae $\omega_0$. 

\begin{definition}\label{canonical outcome}
The set of outcomes $\Omega$ is the set of all maximal consistent sets of formulae $\omega$ such that for each formula $\phi\in\Phi$ if $\N\phi\in \omega_0$, then $\phi\in \omega$.
\end{definition}



Informally, an action of an agent in the canonical game is designed to ``veto" a formula.
The domain of choices of the canonical model consists of all formulae in set $\Phi$. To veto a formula $\psi$, an agent must choose action $\psi$.  The mechanism of the canonical game guarantees that if $\cN\B_C\psi\in \omega_0$ and all agents in the coalition $C$ veto formula $\psi$, then $\neg\psi$ is satisfied in the outcome.

\begin{definition}
The domain of actions $\Delta$ is set $\Phi$.
\end{definition}

\begin{definition}\label{canonical play}
The set $P\subseteq \Delta^\mathcal{A}\times \Omega$ consists of all pairs $(\delta,\omega)$ such that for  any formula $\cN\B_C\psi\in \omega_0$, if $\delta(a)=\psi$ for each agent $a\in C$, then $\neg\psi\in \omega$.
\end{definition}

\begin{definition}\label{canonical pi}
$\pi(p)=\{(\delta,\omega)\in P\;|\; p\in \omega\}$.
\end{definition}

This concludes the definition of the canonical game $G(\omega_0)$. The next four lemmas are auxiliary results leading to the proof of the completeness in Theorem~\ref{completeness theorem}.

\begin{lemma}\label{B child exists lemma}
For any play $(\delta,\omega)\in P$, any action profile $s\in\Delta^C$, and any formula $\neg(\phi\to \B_C\phi)\in \omega$, there is a play $(\delta',\omega')\in P$ such that $s =_C\delta'$ and $\phi\in \omega'$.
\end{lemma}
\begin{proof}
Consider the following set of formulae:
\begin{eqnarray*}
X&\!\!=\!\!&\!\{\phi\}\;\cup\;\{\psi\;|\;\N\psi\in \omega_0\}\\
&&\!\cup\;\{\neg\chi\;|\;\cN\B_D\chi\in \omega_0, D\subseteq C,\forall a\in D(s(a)=\chi)\}.
\end{eqnarray*}
\begin{claim}
Set $X$ is consistent.
\end{claim}
\begin{proof-of-claim}
Suppose the opposite. Thus, there are 
\begin{eqnarray}
\mbox{ formulae }&&\N\psi_1,\dots,\N\psi_m\in \omega_0,\label{choice of psi-s}\\
\mbox{and formulae }&&\cN\B_{D_1}\chi_1,\dots,\cN\B_{D_n}\chi_n\in \omega_0,\label{choice of chi-s}\\
\mbox{such that }&&D_1,\dots,D_n\subseteq C,\label{choice of Ds}\\
&&s(a)=\chi_i\mbox{ for all } a\in D_i, i\le n\label{choice of votes},\\
\mbox{ and }&&\psi_1,\dots,\psi_m,\neg\chi_1,\dots,\neg\chi_n\vdash\neg\phi.\label{choice of cons}
\end{eqnarray}
Without loss of generality, we can assume that formulae $\chi_1,\dots,\chi_n$ are distinct. Thus, assumption~(\ref{choice of votes}) implies that sets $D_1,\dots,D_n$ are pairwise disjoint. 

By propositional reasoning, assumption~(\ref{choice of cons}) implies that
$$
\psi_1,\dots,\psi_m\vdash\phi\to\chi_1\vee\dots\vee\chi_n.
$$
Thus, by Lemma~\ref{super distributivity},
$$
\N\psi_1,\dots,\N\psi_m\vdash\N(\phi\to\chi_1\vee\dots\vee\chi_n).
$$
Hence, by assumption~(\ref{choice of psi-s}),
$$
    \omega_0\vdash\N(\phi\to\chi_1\vee\dots\vee\chi_n).
$$
Thus, by  Lemma~\ref{five plus plus}, using assumptions~(\ref{choice of chi-s}) and the fact that sets $D_1,\dots,D_n$ are pairwise disjoint,
$$
\omega_0\vdash \N(\phi\to\B_C\phi).
$$
Hence $\N(\phi\to\B_C)\in \omega_0$ because set $\omega_0$ is maximal. Then, $\phi\to\B_C\in \omega$ by Definition~\ref{canonical outcome}, which contradicts the assumption  $\neg(\phi\to\B_C)\in \omega$ of the lemma because set $\omega$ is consistent.
Therefore, set $X$ is consistent.
\end{proof-of-claim}

Let $\omega'$ be any maximal consistent extension of set $X$. Thus, $\phi\in X\subseteq\omega'$ by the choice of sets $X$ and $\omega'$. Also, $\omega'\in\Omega$ by Definition~\ref{canonical outcome} and the choice of sets $X$ and $\omega'$.

Let the complete action profile $\delta'$ be defined as follows:
\begin{equation}\label{choice of delta'}
    \delta'(a)=
    \begin{cases}
    s(a), & \mbox{ if } a\in C,\\
    \bot, & \mbox{ otherwise}.
    \end{cases}
\end{equation}
Then, $s=_C\delta'$.

\begin{claim}
$(\delta',\omega')\in P$.
\end{claim}
\begin{proof-of-claim}
Consider any formula $\cN\B_D\chi\in \omega_0$ such that $\delta'(a)=\chi$ for each $a\in D$. By Definition~\ref{canonical play}, it suffices to show that $\neg\chi\in \omega'$. 

\noindent{\bf Case I:} $D\subseteq C$. Thus, $\neg\chi\in X$ by the definition of set $X$. Therefore, $\neg\chi\in \omega'$ by the choice of set $\omega'$.

\noindent{\bf Case II:} $D\nsubseteq C$. Consider any $d_0\in D\setminus C$. Thus, $\delta'(d_0)=\bot$ by equation~(\ref{choice of delta'}). Also, $\delta'(d_0)=\chi$ by the choice of formula $\cN\B_D\chi$. Thus, $\chi\equiv \bot$ and formula $\neg\chi$ is a tautology. Hence, $\neg\chi\in \omega'$ by the maximality of set $\omega'$. 
\end{proof-of-claim}

This concludes the proof of the lemma.
\end{proof}

\begin{lemma}\label{delta exists lemma}
For any outcome $\omega\in\Omega$, there is a complete action profile $\delta\in \Delta^\mathcal{A}$ such that $(\delta,\omega)\in P$.
\end{lemma}
\begin{proof}
Define a complete action profile $\delta$ such that $\delta(a)=\bot$ for each agent $a\in \mathcal{A}$. To prove $(\delta,\omega)\in P$, consider any formula $\cN\B_D\chi\in \omega_0$ 
such that $\delta(a)=\chi$ for each $a\in D$. By Definition~\ref{canonical play}, it suffices to show that $\neg\chi\in \omega$. 

\noindent{\bf Case I}: $D=\varnothing$. Thus, $\vdash\neg\B_D\chi$ by the None to Blame axiom. Hence, $\vdash\N\neg\B_D\chi$ by the Necessitation inference rule. Then, $\neg\N\neg\B_D\chi\notin\omega_0$ by the consistency of the set $\omega_0$. Therefore, $\cN\B_D\chi\notin\omega_0$ by the definition of the modality $\cN$, which contradicts the choice of formula $\cN\B_D\chi$. 

\noindent{\bf Case II}: $D\neq\varnothing$. Thus, there is at least one agent $d_0\in D$. Hence, $\chi=\delta(d_0)=\bot$ by the choice of formula $\cN\B_D\chi$ and the definition of the complete action profile $\delta$. Then, $\neg\chi$ is a tautology. Thus, $\neg\chi\in \omega$ by the maximality of set $\omega$.
\end{proof}

\begin{lemma}\label{N child exists lemma}
For any play $(\delta,\omega)\in P$ and any formula $\neg\N\phi\in \omega$,  there is a play $(\delta',\omega')\in P$ such that $\neg\phi\in \omega'$.
\end{lemma}
\begin{proof}
Consider the set $X=\{\neg\phi\}\;\cup\;\{\psi\;|\;\N\psi\in \omega_0\}$. First, we show that set $X$ is consistent. Suppose the opposite. Thus, there are formulae  $\N\psi_1,\dots,\N\psi_n\in \omega_0$
such that
$
\psi_1,\dots,\psi_n\vdash\phi.
$
Hence, 
$
\N\psi_1,\dots,\N\psi_n\vdash\N\phi
$
by Lemma~\ref{super distributivity}.
Thus, $\omega_0\vdash\N\phi$ because $\N\psi_1,\dots,\N\psi_n\in \omega_0$. Hence,
$\omega_0\vdash\N\N\phi$ by Lemma~\ref{positive introspection lemma}.
Therefore, $\N\phi\in\omega$ by assumption $\omega\in\Omega$ and Definition~\ref{canonical outcome}. Hence, $\neg\N\phi\notin\omega$ by the consistency of set $\omega$, which contradicts the assumption of the lemma. Thus, set $X$ is consistent.

Let $\omega'$ be any maximal consistent extension of set $X$. Note that $\neg\phi\in X\subseteq\omega'$ by the definition of set $X$. By Lemma~\ref{delta exists lemma}, there is a complete action profile $\delta'$ such that $(\delta',\omega')\in P$.
\end{proof}

\begin{lemma}\label{induction lemma}
$(\delta,\omega)\Vdash\phi$ iff $\phi\in\omega$ for each play $(\delta,\omega)\in P$ and each formula $\phi\in\Phi$.
\end{lemma}
\begin{proof}
We prove the lemma by structural induction on formula $\phi$. If $\phi$ is a propositional variable, then the required follows from Definition~\ref{sat} and Definition~\ref{canonical pi}. The cases when $\phi$ is an implication or a negation follow from the maximality and the consistency of set $\omega$ in the standard way.

Suppose that formula $\phi$ has the form $\N\psi$.

\noindent $(\Rightarrow):$ Let $\N\psi\notin\omega$. Thus, $\neg\N\psi\in\omega$ by the maximality of set $\omega$. Hence, by Lemma~\ref{N child exists lemma}, there is a play $(\delta',\omega')\in P$ such that $\neg\psi\in \omega'$. Then, $\psi\notin \omega'$ by the consistency of set $\omega'$. Thus, $(\delta',\omega')\nVdash\psi$ by the induction hypothesis. Therefore, $(\delta,\omega)\nVdash\N\psi$ by Definition~\ref{sat}.

\vspace{1mm}

\noindent $(\Leftarrow):$ Let $\N\psi\in\omega$. Thus, $\neg\N\psi\notin\omega$ by the consistency of set $\omega$. Hence, $\N\neg\N\psi\notin\omega_0$ by Definition~\ref{canonical outcome}. Then, $\omega_0\nvdash \N\neg\N\psi$ by the maximality of set $\omega_0$. Thus, $\omega_0\nvdash\neg\N\psi$ by the Negative Introspection axiom. Hence, $\N\psi\in \omega_0$ by the maximality of set $\omega_0$. Then, $\psi\in \omega'$ for each $\omega'\in\Omega$ by Definition~\ref{canonical outcome}. Thus, by the induction hypothesis, $(\delta',\omega')\Vdash\psi$ for each $(\delta',\omega')\in P$. Therefore, $(\delta,\omega)\Vdash\N\psi$ by Definition~\ref{sat}. 

Suppose that formula $\phi$ has the form $\B_C\psi$. 

\noindent $(\Rightarrow):$ Assume 
that $\B_C\psi\notin \omega$. First, we consider the case when $\psi\notin \omega$. Then, $(\delta,\omega)\nVdash\psi$ by the induction hypothesis. Hence, $(\delta,\omega)\nVdash\B_C\psi$ by Definition~\ref{sat}. 

Next, assume that $\psi\in \omega$. Note that $\psi\to\B_C\psi\notin \omega$. Indeed, if $\psi\to\B_C\psi\in \omega$, then $\omega\vdash \B_C\psi$ by the Modus Ponens inference rule. Thus, $\B_C\psi\in \omega$ by the  maximality of set $\omega$, which contradicts the assumption above.

Because $\omega$ is a maximal set, statement $\psi\to\B_C\psi\notin \omega$ implies that $\neg(\psi\to\B_C\psi)\in \omega$. Thus, by Lemma~\ref{B child exists lemma}, for any action profile $s\in \Delta^C$, there is a play $(\delta',\omega')$ such that $\psi\in \omega'$. Hence, by the induction hypothesis, for any action profile $s\in \Delta^C$ there is a play $(\delta',\omega')$ such that $(\delta',\omega')\Vdash \psi$. Therefore, $(\delta,\omega)\nVdash\B_C\psi$ by Definition~\ref{sat}.

\vspace{1mm}
\noindent $(\Leftarrow):$ Suppose that $\B_C\psi\in \omega$. Thus, $\omega\vdash\psi$ by the Truth axiom. Hence, $\psi\in\omega$ by the maximality of the set $\omega$. Thus, $(\delta,\omega)\Vdash\psi$ by the induction hypothesis.

Next, define an action profile $s\in \Delta^C$ to be such that $s(a)=\psi$ for each $a\in C$. Consider any play $(\delta',\omega')\in P$ such that $s=_C\delta'$. By Definition~\ref{sat}, it suffices to show that  $(\delta',\omega')\nVdash \psi$.

Statement $\B_C\psi\in \omega$ implies that $\neg\B_C\psi\notin \omega$ because set $\omega$ is consistent. Thus, $\N\neg\B_C\psi\notin \omega_0$ by Definition~\ref{canonical outcome} and because $\omega\in\Omega$. Hence, $\neg\N\neg\B_C\psi\in \omega_0$ due to the maximality of the set $\omega_0$. Thus, $\cN\B_C\psi\in \omega_0$ by the definition of modality $\cN$.
Also, $\delta'(a)=s(a)=\psi$ for each $a\in C$. Hence, $\neg\psi\in\omega'$ by Definition~\ref{canonical play} and the assumption $(\delta',\omega')\in P$. Then, $\psi\notin\omega'$ by the consistency of set $\omega'$. Therefore, $(\delta',\omega')\nVdash \psi$ by the induction hypothesis.
\end{proof}

We are now ready to state and prove the strong completeness of our logical system.
\begin{theorem}\label{completeness theorem}
If $X\nvdash\phi$, then there is a game, a complete action profile $\delta$, and an outcome $\omega$ of this game such that $(\delta,\omega)\Vdash\chi$ for each $\chi\in X$ and $(\delta,\omega)\nVdash\phi$.
\end{theorem}
\begin{proof}
Suppose that $X\nvdash\phi$. Thus, set $X\cup\{\neg\phi\}$ is consistent. Let $\omega_0$ be any maximal consistent extension of set $X\cup\{\neg\phi\}$ and $G(\omega_0)=(\Delta,\Omega,P,\pi)$ be the canonical game defined above. Note that $\omega_0\in \Omega$ by Definition~\ref{canonical outcome} and the Truth axiom. 

By Lemma~\ref{delta exists lemma}, there exists a complete action profile $\delta\in \Delta^\mathcal{A}$ such that $(\delta,\omega_0)\in P$. Thus, $(\delta,\omega_0)\Vdash\chi$ for each $\chi\in X$ and $(\delta,\omega_0)\Vdash\neg\phi$ by Lemma~\ref{induction lemma} and the choice of set $\omega_0$. Therefore,  $(\delta,\omega_0)\nVdash\phi$ by Definition~\ref{sat}.
\end{proof}

\section{Conclusion}

In this paper we defined a formal semantics of blameworthiness using the principle of alternative possibilities and Marc Pauly's framework for logics of coalitional power. Our main technical result is a sound and complete bimodal logical system that captures properties of  blameworthiness in this setting. This work is meant to be a step towards formal reasoning about blameworthiness and responsibility.

Recently, there have been several works combining Marc Pauly's  and epistemic logic frameworks to study the interplay between knowledge and know-how strategies~\cite{aa12aamas,aa16jlc,nt17aamas,nt18aaai,nt18ai,nt18aamas} as well as a study of such strategies in a single-agent case~\cite{fhlw17ijcai}. Knowledge is clearly relevant to the study of blameworthiness. Indeed, one can hardly be blamed for not preventing an outcome if one had a strategy to prevent it but did not know what this strategy was. Furthermore, in the legal domain, responsibility is connected to knowledge. For example, US Model Penal Code specifies five types of responsibility based on what the responsible party knew or should have known~\cite{ali62}. In the future, we plan to explore the interplay between knowledge and blameworthiness/responsibility by introducing epistemic component to the framework of this paper.

\label{end of paper}

\end{document}